\documentclass[letterpaper, 10 pt, conference]{ieeeconf}  % Comment this line out if you need a4paper

\IEEEoverridecommandlockouts                              % This command is only needed if 
\usepackage{graphicx}
\usepackage{amsmath}
\usepackage{amsthm}
\usepackage{amssymb}
\usepackage{multicol}
\usepackage{cuted}
\usepackage{multirow}
\usepackage{caption, subcaption}
\newtheorem{theorem}{Theorem}
\newtheorem{lemma}{Lemma}

\newtheorem{assumption}{Assumption}
\usepackage{xcolor}
\usepackage{cite}
\usepackage[colorlinks=true,allcolors=steelblue]{hyperref}
\usepackage{comment}
\usepackage{algorithm}
\usepackage[noend]{algpseudocode}
\usepackage{float}
\definecolor{steelblue}{RGB}{70,130,180}

\def\bbz{\mathbb Z}
\def\bbr{\mathbb R}

\title{\LARGE \bf
Path Planning Using Wassertein Distributionally Robust Deep Q-learning
}

\author{Cem Alpt\"{u}rk and Venkatraman Renganathan% <-this % stops a space
\thanks{This project has received funding from the European Research Council (ERC) under the European Union’s Horizon 2020 research and innovation program under grant agreement No 834142 (Scalable Control). C. Alpt\"{u}rk was a Masters Thesis student at the Department of Automatic Control, LTH, Lund University, Sweden. V. Renganathan is with the Department of Automatic Control, LTH, Lund University, Sweden. Email:        venkatraman.renganathan@control.lth.se, cem.alpturk@gmail.com}%
}

\begin{document}

\maketitle
\thispagestyle{empty}
\pagestyle{empty}

%%%%%%%%%%%%%%%%%%%%%%%%%%%%%%%%%%%%%%%%%%%%%%%%%%%%%%%%%%%%%%%%%%%%%%%%%%%%%%%%
\begin{abstract}
We investigate the problem of risk averse robot path planning using the deep reinforcement learning and distributionally robust optimization perspectives. Our problem formulation involves modelling the robot as a stochastic linear dynamical system, assuming that a collection of process noise samples is available. We cast the risk averse motion planning problem as a Markov decision process and propose a continuous reward function design that explicitly takes into account the risk of collision with obstacles while encouraging the robot's motion towards the goal. We learn the risk-averse robot control actions through Lipschitz approximated Wasserstein distributionally robust deep Q-learning to hedge against the noise uncertainty. The learned control actions result in a safe and risk averse trajectory from the source to the goal, avoiding all the obstacles. Various supporting numerical simulations are presented to demonstrate our proposed approach.    
\end{abstract}

%%%%%%%%%%%%%%%%%%%%%%%%%%%%%%%%%%%%%%%%%%%%%%%%%%%%%%%%%%%%%%%%%%%%%%%%%%%%%%%%

\section{Introduction} \label{sec_intro}
% intro to risk based motion planning
Given the tremendous increase in the computing power, many computationally expensive control theory problems can now be addressed using the deep reinforcement learning approaches \cite{xieDeepRLPath, yan2020towards}. So far, the motion planning problem with uncertainty has been investigated from two different perspectives namely the control theory \cite{luders2010chance} and the reinforcement learning (RL) \cite{base_paper}. When stochastic uncertainties are considered in the problems such as path planning, both the above said approaches resort to the powerful stochastic optimization techniques as in \cite{kandel2020safe} to ensure satisfaction of specifications with high probability. However, when assumptions of certain functional forms for the system uncertainties are made in the name of tractability, they may lead to potentially severe miscalculation of risk when the uncertain robot is made to operate in a dynamic environment \cite{majumdarRisk}. Such shortcomings can be addressed through carefully designed risk bounded motion planning approaches using distributionally robust optimization techniques. The interested readers are referred to these non-exhaustive list of papers on risk averse motion planning \cite{aoude2013probabilistically, subramani2019risk, xiao2019explicit, lathrop2021distributionally, sleimanRisk}. 

% intro to deep Q learning
% intro to safe/robust RL
Risk averse path planning problems emphasize the need for exact propagation of uncertainties. For instance, either the distributions of all the uncertainties or the moments defining the distributions are required to be known in advance or calculated exactly for all time steps to evaluate the risk of obstacle collision as in \cite{renganathan2022risk}. It is an usual practice to associate a particular distribution to the uncertainty (often Gaussian) just for the sake of tractability \cite{luders2010chance}. But often in reality, all we have is just a collection of samples of the uncertainty and trying to fit a distribution to it may cause undue risk. On a parallel note, the central idea of safe and robust RL as described in \cite{morimoto2005robust, brunke2022safe} is to learn control policies for agents that encourage safety or robustness, and to design methods that can formally certify the safety of a learned control policy. For instance, a maximum entropy based lower bound on a robust RL objective was used to learn policies that are robust to some disturbances in the dynamics and the reward function in \cite{eysenbach2021maximum}. But analysis of safe and robust RL algorithms with distributional uncertainty has received very less attention. Authors in \cite{kandel2020safe} use the Wasserstein distributionally robust deep Q-learning to hedge against the distributional uncertainty and approximately solve the Bellman equation associated with the deep Q-learning approach given in \cite{dqn_atari}. In this paper, we stick to the sample based uncertainty modeling of process noise and take a similar approach as \cite{kandel2020safe}, and further use the Lipschitz constant based approximations advocated in Theorem 5 of \cite{kuhn} to learn risk-averse robot control actions. A similar problem was investigated by \cite{base_paper}, albeit with usual Gaussian assumptions and no formal risk consideration.  

\emph{Contributions:} This article leverages powerful results in deep reinforcement learning theory and distributionally robust optimization to learn control policies for robots to operate in a risk-averse manner in an environment. Our main contributions are: 
\begin{enumerate}
    \item we learn safe robot control actions at all the state space positions to infer a trajectory to move from source to goal by avoiding all obstacles. We account for the uncertainty due the robot initial states and the process noise through reward function design and learn the risk averse control actions using approximated Wasserstein distributionally robust $Q$-learning.
    \item we demonstrate our proposed approach using a series of numerical simulations and show the effectiveness of our proposed approach. 
\end{enumerate}
Following a short summary of notations and preliminaries, the rest of the paper is organized as follows. In \S \ref{sec_problem_formulation}, the risk-averse path planning problem associated with the uncertain robot system is presented. The Wasserstein distributionally robust $Q$-learning approach is discussed in \S \ref{sec_wdrql}. The proposed idea is then demonstrated using a numerical simulation in \S \ref{sec_num_results}. Finally, the paper is closed in \S \ref{sec_conclusion}. Due to the page restrictions, some proofs are available in the appendix. 
\section*{Notations \& Preliminaries}
The set of real numbers, integers are denoted by $\bbr, \bbz$. The subset of real numbers greater than $a \in \bbr$ is denoted by $\bbr_{> a}$.
The set of integers between two values $a,b \in \bbz$ with $a<b$ is denoted by $[a:b]$. 
The set of non-negative integers is denoted by $\mathbb{Z}_{+}$.
We denote by $\mathcal{B}(\bbr^{d})$ and $\mathcal{P}(\bbr^{d})$ the Borel $\sigma-$algebra on $\bbr^{d}$ and the space of probability measures on $(\bbr^{d}, \mathcal{B}(\bbr^{d}))$ respectively. A probability distribution with mean $\mu$ and covariance $\Sigma$ is denoted by $\mathbb{P}(\mu, \Sigma)$, and specifically $\mathcal{N}_{d}(\mu, \Sigma)$ if the distribution is normal in $\mathbb{R}^{d}$. An uniform distribution over a set $A$ is denoted by $\mathcal{U}(A)$. Given a constant $q \in \bbr_{\geq 1}$, the set of probability measures in $\mathcal{P}(\bbr^{d})$ with finite $q-$th moment is denoted by $\mathcal{P}_{q}(\bbr^{d}) := \left\{ \mu \in \mathcal{P}(\bbr^{d}) \mid \int_{\bbr^{d}} \left \Vert x \right \|^{q} d\mu < \infty \right\}$. The type-$q$ Wasserstein distance $\forall q \geq 1$ between $\mathbb{Q}_1, \mathbb{Q}_2 \in \mathcal{P}_{q}(\bbr^{d})$ with $\Pi(\mathbb{Q}_1, \mathbb{Q}_2)$ being the set of all joint probability distributions on $\bbr^{d} \times \bbr^{d}$ with marginals $\mathbb{Q}_1$ and $\mathbb{Q}_2$ is
\begin{equation} \label{eqn_wass_distance}
    \footnotesize W_{q}(\mathbb{Q}_1, \mathbb{Q}_2) \overset{\Delta}{=} \left( \underset{\pi \in \Pi(\mathbb{Q}_1, \mathbb{Q}_2)}{\inf}  \int_{\bbr^{d} \times \bbr^{d}} \left \Vert z_1 - z_2 \right \|^{q} \pi(d z_{1}, d z_{2}) \right)^{\frac{1}{q}}. 
\end{equation}
\section{Problem Formulation} \label{sec_problem_formulation}
\subsection{Robot \& Environment Model}
The robot is modeled as a stochastic discrete time linear time invariant system and it is assumed to move within a bounded environment $\mathcal{X} \subset \mathbb{R}^{n_{x}}$. 
There are in total $M \in \mathbb{Z}_{+}$ obstacles in the environment, each disjoint with the other and they are collectively referred as $\mathcal{O}$ with $\left | \mathcal{O} \right \vert = M$. Further, each obstacle is assumed to be static and of convex polytope shape. Then, the free space that the robot can traverse namely $\mathcal{X}_{\mathbf{free}} \subset \mathcal{X}$ is given by
\begin{align}
\mathcal{X}_{\mathbf{free}} = \mathcal{X} \, \backslash \, \mathcal{X}_{\mathbf{obs}}, \quad \text{and} \quad \mathcal{X}_{\mathbf{obs}} := \bigcup \limits_{i=1}^M \mathcal{X}_{\mathbf{obs}}^{(i)},
\end{align}
where $\mathcal{X}_{\mathbf{obs}}^{(i)} \subset \mathcal{X}$ is the space occupied by the obstacle $i \in \mathcal{O}$. 
Similar to the obstacles, we define a goal region, $\mathcal{X}_{\mathbf{goal}} \subset \mathcal{X}$, that is both static and circular in shape with constant radius $R_{\mathbf{goal}} > 0$. This is a fair assumption \footnote{Our problem formulation works perfectly fine even with convex polytopic goal regions too.} given that all the robot states that are inside the region $\mathcal{X}_{\mathbf{goal}}$ which is centered at the goal point $\bar{x}_{\mathbf{goal}} \in \mathcal{X}_{\mathbf{goal}}$ are considered to be goal states. The position of the robot at time $k \in \mathbb{Z}_{+}$ is denoted as $p_{r,k} \in \mathbb{R}^{n_{r}}$. The robot is limited to move within the environmental boundaries whose limits are $[\underline{p}, \overline{p}]$ with $\underline{p}, \overline{p} \in \mathbb{R}^{n_{r}}$. Hence, just like the obstacles, the environmental boundaries are also treated as terminal states. The state of the robot at time $k$ is represented as $x_{k} \in \mathbb{R}^{n_{x}}$ and it may include the robot's position, velocity and other states of interest so that $n_{x} \geq n_{r}$. The robot is controlled through a control input $u_k$ which is selected from $\mathcal{U}$ such that $u_k \in \mathcal{U} \subseteq \mathbb{R}^{n_{u}}$. 
Given the above description, we define the dynamics (evolution) of the robot in $\mathcal{X}$ as
\begin{equation} \label{eqn_robot_dynamics}
    x_{k+1} = Ax_k + Bu_k + w_k.
\end{equation}
The robot is subject to a process disturbance $w_k \in \mathbb{R}^{n_{x}}$. The time invariant true distribution of the process noise $w_k$ at any time $k$ namely $\mathbb{P}_w$ is unknown, however it is assumed that a collection of $N \in \mathbb{N}$ independent samples of $w_k$ are available beforehand. 
That is, an i.i.d. sequence $\hat{w}_{1},\dots,\hat{w}_{N} \in \mathbb{R}^{n_{r}}$ is assumed to be known in advance. However, at any time $k$, the distribution of $w_{k}$ can be approximated through the following empirical distribution, $\hat{\mathbb{P}}_{w} = \frac{1}{N} \sum \limits_{i=1}^N \delta_{\hat{w}_i}$, where $\delta_{w_i}$ is the Dirac delta function. Note that, $\hat{\mathbb{P}}_{w}$ need not necessarily be the true distribution of the $w_{k}$. This is precisely where our approach differs from the existing safe RL literature, where it is a common practice to either assume a distribution for $w$ or bound for $w$. The initial state of the robot is assumed to be random and it is modelled as $x_0 \sim \mathbb{P}_{x_{0}} \left(\bar{x}_{0}, \Sigma_{x_0}\right)$, where $\mathbb{P}_{x_{0}}$ is assumed to be known with the mean $\bar{x}_{0} \in \mathbb{R}^{n_{x}}$, and the covariance $\Sigma_{x_0} \in \mathbb{R}^{n_{x} \times n_{x}}$ also being assumed to be known or estimated from prior experiments. It is clear from the above setting that $\mathbb{P}_{x_{k}}$ for $k \geq 1$ is not known exactly despite $\mathbb{P}_{x_{0}}$ being known exactly. 

\begin{assumption} \label{assume:obs_goal_separate}
There exists a minimum separation distance $L_{min} > 0$ between the goal and any of the obstacle regions. That is, $\mathcal{X}_{\mathbf{goal}} \cap \mathcal{X}_{\mathbf{obs}} = \emptyset$ and $\forall x_{\mathbf{goal}} \in \mathcal{X}_{\mathbf{goal}}, \forall x_{\mathbf{obs}} \in \mathcal{X}_{\mathbf{obs}}$, we see that
\begin{align} \label{eqn_min_separation}
    \left \| x_{\mathbf{goal}} - x_{\mathbf{obs}} \right \Vert_{2} \geq L_{min}.
\end{align}
\end{assumption}

% \begin{figure}[h]
%     \centering
%     \includegraphics[scale=0.5]{figures/environment.png}
%     \caption{The environment consists of a goal region in green color and two obstacles in red color. The robot's initial position is shown as a green point and its sample trajectory in shown in red color.}
%     \label{fig:environment}
% \end{figure}

\noindent \textbf{Main Problem Statement:} Given the uncertain robot evolution as in \eqref{eqn_robot_dynamics} with sample based process noise model, we learn the risk-averse control policy for all state space positions of the robot and hence design a trajectory for the robot from its initial state $x_{0}$ to the goal region $\mathcal{X}_{\mathbf{goal}}$ without colliding with any of the obstacles $\mathcal{O}$.
\subsection{Markov Decision Process (MDP) Formulation} \label{sec_mdp}
Given that we have to learn what actions to take provided we land anywhere in $\mathcal{X}$ given the uncertainty in the distributional information of $w$, taking control theory perspective can be hard. Hence, we resort to the RL approaches to address this shortcoming. That is, the above planning problem can be cast as a Markov decision process that consists of the tuple $\langle\mathcal{S}, \mathcal{A}, \mathbb{P}, r \rangle$. 
Here,  $\mathcal{S} \subset \mathbb{R}^{n_{\mathcal{S}}}$ is the state space, $\mathcal{A} \subset \mathbb{R}^{n_{\mathcal{A}}}$ is a finite set called the action space with $|\mathcal{A}| \in \mathbb{N}_{+} \, \backslash \, 0$, $r: \mathcal{S} \to \mathbb{R}$ is the reward function and $\mathbb{P}: \mathcal{S} \times \mathcal{A} \to \mathcal{P}(\mathcal{S})$ is the state transition probability which defines the probability distribution over the next states. 
We denote by $\hat{a}$, the null action where it does not cause a change in the position of the robot. 
We denote the total action space as $\hat{\mathcal{A}} = \mathcal{A} \cup \hat{a}$. 
Due to the Markov property of the system, the transition probabilities only depend on previous state and action such that for a given state $s_k$ and action $a_k$ and the history $h_k = \{s_0,a_0,...,s_k,a_k \}$, we see that $\mathbb{P}(s_{k+1} \mid h_k) = \mathbb{P}(s_{k+1} \mid s_k,a_k), \forall s_{k+1} \in \mathcal{S}$. 
At step $k$, the state $s_k \in \mathcal{S}$ contains the state of the robot $x_k$, the center of the goal $p_g$, and the centers of the obstacles $p_{\mathbf{obs}}^{(i)}$.
\begin{align}
    s_k := \left \{x_k, p_g, p_{\mathbf{obs}}^{(i)} \right\} \in \mathbb{R}^{n_{\mathcal{S}}}, 
\end{align}
where, $n_{\mathcal{S}} = (2+M)n_{x}$, and $i=1,\dots,M$. The state $s_{k}$ is referred as a \emph{terminal state} if $x_{k} \in \mathcal{X}_{\mathbf{obs}} \cup \mathcal{X}_{\mathbf{goal}}$ or if $x_{k} \notin \mathcal{X}$ and as \emph{non-terminal state} otherwise. The dimensions of MDP state $s_{k}$ depend on the number of obstacles in the environment. Increasing the number of obstacles will cause the dimension of $s_k$ to increase as well \footnote{The increase in the dimension of $s_{k}$ is the price that we need to pay to handle potentially dynamic obstacles.}.
An action $a_k \in \hat{\mathcal{A}}$ performed at state $s_k \in \mathcal{S}$, will cause a transition to a new state $s_{k+1} \in \mathcal{S}$ with the probability $\mathbb{P}(s_{k+1} \mid s_k, a_k)$. 
After the transition, a deterministic reward $r_k = r(s_{k+1})$ is obtained based on the state where we land. 
The actions, $a_k = \pi(s_k)$ are chosen based on a deterministic policy $\pi: \mathcal{S} \rightarrow \hat{\mathcal{A}}$. 
The set of all admissible control policies is denoted by $\Pi$. 
A sequence $\tau_k^{(\pi)} := (s_k,a_k,s_{k+1},a_{k+1},\dots,s_{K-1},a_{K-1},s_K)$ with a terminal state $s_K$ is called a sample path under the policy $\pi \in \Pi$. 
The cumulative discounted reward for this sample path is 
\begin{align}
    R^{\pi}(\tau_k) = \sum \limits_{i=0}^{K-k-1}\gamma^{i}r(s_{i+k+1}),
\end{align}
where $\gamma \in [0,1]$ is the discount factor. The discount factor is used in order to take in to account the future rewards. The value function is defined as the expected value calculated for the discounted returns starting from state $s$ and following policy $\pi$ and the Q-function is defined as the expected discounted return if action $a$ is taken at state $s$ and following policy $\pi$. 
That is,
\begin{align}
V^{\pi}(s) &= \mathbb{E}\left[R^{\pi}(\tau_k \mid s_k=s)\right], \quad \text{and}\\ 
Q^{\pi}(s,a) &= \mathbb{E}\left[\sum \limits_{k=0}^{K-1} \gamma^k r(s_{k+1}) \mid s_0 = s, a_0 = a\right].
\end{align}
The Q values for a state determine what action is the best to take. 
Hence, the \emph{modified} policy $\pi$, tailored for this path planning problem is related to the Q-function as
\begin{align}
    \pi(s) &= \begin{cases} 
    \underset{{a \in \mathcal{A}}}{\arg \max} \, Q^{\pi}(s,a), &\text{if } s \notin \text{terminal state} \\
    \hat{a}, &\text{if } s \in \text{terminal state}
    \end{cases}.
\end{align}

\subsection{Reward Function Design \& Its Approximation}
In this work, we consider rewards that depend only on $s_{k+1}$ and it includes a penalty for both traveling and collision with obstacles along with an incentive for being in the goal. That is,
\begin{align} \label{eqn_discont_reward}
    \hat{r}(s^{\prime}) = \mathbf{r_{travel}} + 
    \begin{cases}
        \mathbf{r_{goal}}, & \quad \text{if } s^{\prime} \in \mathcal{X}_{\mathbf{goal}}, \\
        \mathbf{r_{obs}}, & \quad \text{if } s^{\prime} \in \mathcal{X}_{\mathbf{obs}} \text{ or } s^{\prime} \notin \mathcal{X}
    \end{cases},
\end{align}
where $\mathbf{r_{travel}}$ is the travel penalty, $\mathbf{r_{goal}}$ is the reward for reaching the goal, and $\mathbf{r_{obs}}$ is the penalty for obstacle collision. 
The discontinuity in the reward function $\hat{r}(\cdot)$ due to the switching in \eqref{eqn_discont_reward} causes its Lipschitz constant $L_{\hat{r}} \to \infty$ when we approximate the $Q$-function later on using a neural network. Hence, it has to be approximated by a Lipschitz continuous function $r(s^{\prime})$. The radial step function associated with switching to the goal reward can be approximated as
\begin{align}
    f_{\mathbf{goal}}(p_{r}) &= \frac{\mathbf{r_{goal}}}{2}\left(1 + \tanh \left(\frac{d_{2}(p_{r}, \bar{p}, R_{\mathbf{goal}}) }{\delta}\right) \right),
\end{align}
where $p_{r}, \bar{p} \in \mathbb{R}^{n_{r}}$ denote the positions of the robot and the center of the goal respectively with $d_{2}(p_{r}, \bar{p}, R_{\mathbf{goal}}) = R_{\mathbf{goal}} - \left \| p_{r} - \bar{p} \right\Vert_{2}$ being the distance function and $\delta \in \mathbb{R}_{>0}$ is the slope. A similar structure, $f_{\mathbf{bor}}(p_{r})$ can be used for the borders that takes the distance to the borders defined using limits $[\underline{p}, \overline{p}]$ across all the position dimensions. The step function associated with switching to the obstacle collision penalty will have the convex polytope shape as its support. Let $\mathbf{q} := \{q_{i}\}^{n_r}_{i=1}$, with each $q_i$ being a large, positive and even integer. Then, using the modified distance function \footnote{For $n_{r} \geq 2$, the distance function should be defined using the level set of the convex obstacle obtained from its compact support and smooth approximation can be done using the $\tanh$ (or logistics) function in $\mathbb{R}^{n_{r}}$.}
\begin{align}
d_{\mathbf{q}}(p_{r}, \bar{p}, R_{\mathbf{obs}}) = R_{\mathbf{obs}} - \left(\sum^{n_{r}}_{i=1} \left| p^{(i)}_{r} - \bar{p}^{(i)} \right \vert^{q_{i}} \right)^{\frac{1}{\max\{\mathbf{q}\}}},   
\end{align}
with $R_{\mathbf{obs}} > 0$, we can obtain a smooth approximation of a rectangle whose centroid is at $\bar{p}$ and length of its biggest side being $2R_{\mathbf{obs}}$. Then, the polytopic obstacle $i \in \mathcal{O}$ can be represented by stitching together several such rectangles defined using the tuple $\{\bar{p}_{j}, R^{(j)}_{\mathbf{obs}}\}^{M_{i}}_{j = 1}, M_{i} \in \mathbb{N}_{\geq 2}$. Hence,
\begin{align}
    f^{(i)}_{\mathbf{obs}}(p_{r}) = \frac{\mathbf{r_{obs}}}{2}\left(1 + \tanh \left(\frac{ \sum^{M_{i}}_{j = 1} d_{\mathbf{q}}(p_{r}, \bar{p}_{j}, R^{(j)}_{\mathbf{obs}}) }{\delta}\right) \right),
\end{align}
where, $i \in \mathcal{O}$. Then, the Lipschitz continuous approximation $r(s^{\prime})$ of the original reward function $\hat{r}(s^{\prime})$ is given by
\begin{align} \label{eqn_continuous_reward_fn}
    r(s^{\prime}) = \mathbf{r_{travel}} + f_{\mathbf{goal}}(p_{r}) + f_{\mathbf{bor}}(p_{r}) + \sum_{i \in \mathcal{O}} f^{(i)}_{\mathbf{obs}}(p_{r}).
\end{align}
\section{Learning Risk-Averse Control Actions} \label{sec_wdrql}
If the Q-values for a system is known, a policy $\pi$ can be used to maximize the expected returns. 
In order to estimate the Q-values, the standard temporal difference learning based Q-Learning procedure is usually employed, \cite{tabular_q}. From now on, we drop the superscript $\pi$ on $Q^{\pi}(s,a)$ for the brevity of notation. We now define the Bellman operator, $\mathcal{T}: \mathbb{R}^{\mathcal{S} \times \hat{\mathcal{A}}} \to \mathbb{R}^{\mathcal{S} \times \hat{\mathcal{A}}}$ as
\begin{equation} \label{eqn_Bellman}
    \mathcal{T}Q(s,a) = \mathbb{E}_{s^{\prime}}\left[ r(s^{\prime}) + \gamma \max_{a^{\prime} \in \mathcal{A}} Q(s^{\prime}, a^{\prime}) \right], 
\end{equation}
where the outer expectation is over the next states $s^{\prime}$, which come from the transition probability $\mathbb{P}(s^{\prime} \mid s,a)$. Given the continuous state space setting, we propose to use the Deep Q Learning (DQN) approach as in \cite{dqn_atari}, which estimates the Q values by using a deep neural network. Specifically, it utilizes two neural networks namely: i) Q-network $Q(s,a,\theta)$, and ii) the target network $Q(s,a,\theta^-)$. The Q network is trained by experience replay, where a random batch of experiences are sampled from the memory buffer and with the Bellman equation, the targets are calculated. For an experience $ \langle s, a, r, s^{\prime} \rangle$, the target $y$ is calculated as
\begin{equation}
    y = 
    \begin{cases}
        r + \gamma \,  \underset{a^{\prime}\in \mathcal{A}}{\max} \,  Q(s^{\prime},a^{\prime}, \theta^-), \quad & \text{ if } s^{\prime} \text{ is non-terminal,} \\
        r, \quad & \text{ if } s^{\prime} \text{ is terminal}.
    \end{cases}
\end{equation}
\subsection{The Lipschitz Approximated Wasserstein Distributionally Robust Deep Q-Learning} \label{subsec_wass_q_learning}
The Q-function estimation defined in section \ref{sec_mdp} will be formulated as a distributionally robust optimization problem. Taking an action $a$ at state $s$ causes the transition to an unknown state $s^{\prime}$ with unknown distribution $\mathbb{P}_{s^{\prime}} \triangleq \mathbb{P}(s^{\prime} \mid s,a)$. For brevity, $a$ will be omitted in the notation from now.
\begin{assumption} \label{assump:light_tail}
The true distribution $\mathbb{P}_{s^{\prime}}$ is a light-tailed distribution \cite{dr_learning}. That is, $\exists p > 1$ such that,
\begin{equation}
    \mathbb{E}_{\mathbb{P}_{s^{\prime}}} \left[ e^{\|s^{\prime}\|^{p}}\right] = \int_{\mathcal{S}} e^{\|s^{\prime}\|^{p}}d\mathbb{P}_{s^{\prime}}(s^{\prime}) < \infty.
\end{equation}
\end{assumption}

\subsubsection{The Wasserstein Ambiguity Set}
Given a robot state $x_k \in \mathcal{X}$ and an input $a_k \in \hat{\mathcal{A}}$, an empirical distribution for $x_{k+1}$ is given by,
\begin{equation}
    \hat{\mathbb{P}}_{x_{k+1}} := \frac{1}{N} \sum \limits_{i=1}^N \delta_{\hat{x}^{(i)}_{k+1}} = \frac{1}{N} \sum \limits_{i=1}^N \delta_{Ax_k + Ba_k + \hat{w}_i}.
    \label{eq:px}
\end{equation}
By knowing the state of the robot $x_k$, the full state $s_k$ can be obtained by using the positions of the goal and obstacles of the current environment (which do not depend on the position of the robot), since these stay constant during an episode.
For ease of notation, we refer to the next state $s_{k+1}$ as $s^{\prime}$, and the samples of $s^{\prime}$ obtained from (\ref{eq:px}), are denoted as $\hat{s}^{\prime (i)}$ for $i=1,\dots N$. Then, the empirical distribution is given by $\hat{\mathbb{P}}_{s^{\prime}} = \frac{1}{N} \sum \limits_{i=1}^N \delta_{\hat{s}^{\prime (i)}}$. When an action $a$ is performed while in state $s$, the nominal distribution for the center of the Wasserstein ball will be $\hat{\mathbb{P}}_{s^{\prime}}$, and the worst case transition will be coming from a distribution that is inside this ball. 
We define the ambiguity set $\mathcal{B}_{s,a}$,
\begin{equation}
    \mathcal{B}_{s,a} := \left \{ \mathbb{P}_{s^{\prime}} \in \mathcal{P}(\mathcal{S}) \mid W_1 \left(\hat{\mathbb{P}}_{s^{\prime}}, \mathbb{P}_{s^{\prime}}\right) < \epsilon_{s^{\prime}} \right\}.
\end{equation}
The Wasserstein ball radius $\epsilon_{s^{\prime}}$ is chosen such that the true distribution $\mathbb{P}_{s^{\prime}}$ lies within this Wasserstein ball with probability greater than $1-\beta$. The $\beta$ parameter will determine the allowed risk factor for the solution. A smaller $\beta$ will result in a larger radius which causes the generated policy to be much more risk averse and vice-versa. Since, the radius $\epsilon_{s^{\prime}}$ quantifies the amount of trust (distrust) that we have over the $\hat{\mathbb{P}}_{s^{\prime}}$, it is chosen such that 
\begin{equation}
    \mathbb{P}\left( \mathbb{P}_{s^{\prime}} \in \mathcal{B}_{s,a} \right) \geq 1 - \beta , \quad \beta \in [0,1].
\end{equation}
\begin{lemma}
Based on Assumption \ref{assump:light_tail}, for an empirical distribution $\hat{\mathbb{P}}_{s^{\prime}}$ with $N$ atoms and $\rho = \mathbf{diam}(\mathbf{supp}(\hat{\mathbb{P}}_{s^{\prime}}))$, the radius of the Wasserstein ambiguity set for the state distributions is
\begin{equation} \label{eqn_wasserstein_ball_radius}
    \epsilon_{s^{\prime}} = \rho \sqrt{ \frac{2}{N} \ln \left( \frac{1}{\beta}\right)}.
\end{equation} 
\label{lem:radius}
\end{lemma}

\subsubsection{Approximated Solution to The Wasserstein Distributionally Robust $Q$-learning Problem}
We define the distributionally robust Bellman operator $\hat{\mathcal{T}}: \mathbb{R}^{\mathcal{S} \times \hat{\mathcal{A}}} \to \mathbb{R}^{\mathcal{S} \times \hat{\mathcal{A}}}$ to represent the worst case expected returns, so that risk can be incorporates into the Q-values. That is, 
\begin{align} \label{eq:dr_bellman}
    \hat{\mathcal{T}} Q(s,a) &:= \inf_{\mathbb{P}_{s^{\prime}} \in \mathcal{B}_{s,a}} \mathbb{E}_{s^{\prime} \sim \mathbb{P}_{s^{\prime}}} \left[ h(s^{\prime}) \right], \quad \text{where}, \\
    h(s^{\prime}) &:= \underbrace{r(s^{\prime})}_{:= h_{r}(s^{\prime})} + \underbrace{\gamma \max_{a^{\prime} \in \mathcal{A}}Q(s^{\prime}, a^{\prime})}_{:= h_{Q}(s^{\prime})} \label{eqn_hs}.  
\end{align}
Since the Q-function is approximated using a neural network with hidden layers and non-linear activation functions, $h(s^{\prime})$ turns out to be a non-convex function of the states. Since an exact solution to the infinite dimensional problem \eqref{eq:dr_bellman} using duality theory is difficult to find when the objective function is non-convex, we resort to the Lipschitz constant based approximation. 
\begin{lemma} \label{lem:lip_min}
The Lipschitz approximation of the right hand side of \eqref{eq:dr_bellman} has an equivalent solution for the case when the objective function $h$ is to be minimized and this results in a lower bound for (\ref{eq:dr_bellman}), where $\epsilon_{s^{\prime}}$ becomes larger in practice and the solution can become more risk averse. That is,
\begin{align}
    \inf_{\mathbb{P}_{s^{\prime}} \in \mathcal{B}_{s,a}} \mathbb{E}_{s^{\prime}} [h(s^{\prime})] \geq \mathbb{E}_{s^{\prime}} [h(s^{\prime})] - \epsilon_{s^{\prime}} L_h.
    \label{eq:dro_approx}
\end{align}
\end{lemma}

\subsubsection{Calculating the Lipschitz Constant $L_h$ of $h(s^{\prime})$.}
The Lipschitz constant for $h_{r}(s^{\prime})$ and $h_{Q}(s^{\prime})$ can be calculated or estimated independently and then combined to get the Lipschitz constant of $h(s^{\prime})$. The second part of $h(s^{\prime})$ given by $h_{Q}(s^{\prime})$ contains the Q-function which is approximated by a neural network. The neural network takes the state $s$ as an input and returns the Q-values for each action. An upper bound for the Lipschitz constant of a dense neural network with ReLU activation functions can be approximated using the LipSDP package developed by \cite{LipSDP}.
\begin{lemma} \label{lem:lip_sum}
Let $f_i: \mathbb{R}^n \to \mathbb{R}, i = [1:N]$ be Lipschitz continuous with constants $K_{f_{i}}$. Then, the Lipschitz constant of the functions $f_{g} = \max \left( \left\{ f_i(x) \right\}^{N}_{i=1} \right)$ and $f_{f} = \sum^{N}_{j=1} f_{j}(x)$ are respectively 
\begin{align} \label{lem:lip_q}
    K_{f_{g}} = \max\left\{K_{f_1},\dots,K_{f_N}\right\}, \quad \text{and} \,  K_{f_{f}} = \sum^{N}_{j=1} K_{f_{i}}.
\end{align}
\end{lemma}
\begin{lemma}
Let $A \in \mathbb{R}$, $\delta \in \mathbb{R}_{> 0}$ and $p,g \in \mathbb{R}^2$. Then, the Lipschitz constants of the scalar functional $F(x) = \frac{A}{2}\left(1 + \tanh \left(\frac{x}{\delta} \right) \right)$, the function $f(p) = \frac{A}{2}\left(1 + \tanh\left(\frac{R - \|p - g\|_2}{\delta}\right) \right)$, and the function $\mathcal{K}(p) = \frac{A}{2}\left(1 + \tanh\left(\frac{d_{q}(p, \bar{p}, R)}{\delta}\right) \right)$ are equal and $K_f = L_f = \mathcal{K}_{f} = \frac{|A|}{2\delta}$ respectively. 
\label{lem:lip_tanh}
\end{lemma}

\begin{theorem}
Given assumption \ref{assume:obs_goal_separate}, the Lipschitz constant of the reward function $r$ given by \eqref{eqn_cont_reward} is 
\[L_{r} = \frac{\max \{|\mathbf{r_{goal}}|,|\mathbf{r_{obs}}| \}}{2\delta}.\]
\end{theorem}
\begin{proof}
Based on Assumption \ref{assume:obs_goal_separate}, the individual terms that contribute to the total reward function $r$ in \eqref{eqn_cont_reward} do not interfere with each other. 
Then, it follows from Lemma \ref{lem:lip_sum} that the Lipschitz constant of the reward function is the maximum of the Lipschitz constants of the individual terms.
\end{proof}

\begin{lemma} \label{lem:lip_max}
Given that $h_{Q}(s^{\prime}) = \gamma \max \left\{Q\left(s^{\prime},a_1 \right), \dots, Q\left(s^{\prime},a_{n_{\mathcal{A}}}\right)\right\}$, where $n_{\mathcal{A}} = |\mathcal{A}|$, and the network has the upper bounded Lipschitz constants $K_{a_i}, \forall a_i \in \mathcal{A}$, the Lipschitz constant of $h_{Q}(s^{\prime})$ is
\begin{equation} \label{lem:q}
    L_Q = \gamma \max \left \{K_{a_1},\dots,K_{a_{n_{\mathcal{A}}}} \right\}.
\end{equation}
\end{lemma}

Proof of the above lemma follows by direct application of Lemma \ref{lem:lip_sum} on $h_{Q}(s^{\prime})$. Having found the Lipschitz constants of both $h_{r}(s^{\prime})$ and $h_{Q}(s^{\prime})$, the following theorem establishes the Lipschitz constant for the objective function $h(s^{\prime})$. 

\begin{theorem}
The upper bound of the Lipschitz constant of $h(s^{\prime})$ defined in \eqref{eq:dr_bellman} is given by 
\begin{align}
    L_{h} \leq \frac{\max\{|\mathbf{r_{goal}}|,|\mathbf{r_{obs}}|\}}{2\delta} + \gamma \max \left\{K_{a_1},\dots,K_{a_{n_{\mathcal{A}}}}\right\}.
\end{align}
\end{theorem}
\begin{proof}
Using \eqref{eqn_hs}, the upper bound for the Lipschitz constant of $h(s^{\prime})$ can be computed by using Lemma \ref{lem:lip_sum} as,
\begin{equation}
    L_h \leq \frac{\max\{|\mathbf{r_{goal}}|,|\mathbf{r_{obs}}|\}}{2\delta} + \gamma \max \left\{K_{a_1},\dots,K_{a_{n_{\mathcal{A}}}}\right\}.
\end{equation}
The result is an upper bound due to $L_{Q}$ being an upper bound to the $h_{Q}(s^{\prime})$.
\end{proof}
% Pseudocode for DRDQN
\begin{algorithm}
    \caption{Lipschitz Approximated Wasserstein Distributionally Robust DQN}
    \begin{algorithmic}
        \Require Disturbance samples $\hat{w}_1,\dots \hat{w}_N$, Learning rate $\eta$, Max episodes $N_{ep}$, Episode Length $N_{step}$, Batch size $N_{batch}$
        \State Initialize replay memory $\mathcal{M} \gets \emptyset$
        \State Initialize network weights $\theta,\theta^{-}$
        \State Estimate Lipschitz constant of network $\theta^-$
        \State Compute $\epsilon_s$ by (\ref{eqn_wasserstein_ball_radius})
        \For{episode $ = 1:N_{ep}$}
            \State Initialize $s \in \mathcal{S}$
            \For{$k=1:N_{step}$}
                \State Select action $a_k$ with $\epsilon$-greedy policy $\pi$
                \State Observe next state $s^{\prime}$ and reward $r$
                \State Append experience $(s,a,r,s^{\prime})$ to $\mathcal{M}$
                \State Initialize loss $\delta \gets 0$
                \For{$j=1:N_{batch}$}
                    \State Sample experience $(s_j,a_j,r_j,s_j^{\prime})$ from $\mathcal{M}$
                    \State Compute nominal distribution $\hat{\mathbb{P}}_{s^{\prime}}$
                    \State Approximate target $y_j$ by (\ref{eq:dro_approx}) with the target network
                    %\State Compute td-error $\delta_j \gets y_j - Q(s_j,a_j;\theta)$
                    \State Accumulate loss $\delta \gets \delta + (y_j - Q(s_j,a_j;\theta))^2$
                \EndFor
                \State Update weights $\theta$ by loss $\delta$ with backpropagation
                \State Set $\theta^- \gets \theta$ and compute Lipschitz constant of network $\theta^-$ every $\Gamma$ steps
            \EndFor
        \EndFor
    \end{algorithmic}
    \label{alg:drq_learning}
\end{algorithm}
\section{Numerical Results} \label{sec_num_results}
% \begin{figure}[h]
%     \centering
%     \includegraphics[width=0.25\linewidth]{figures/actions.png}
%     \caption{Actions}
%     \label{fig:actions}
% \end{figure}
We consider the robot to be moving in an environment $\mathcal{X} \subset \mathbb{R}^{2}$ with the limits of $\mathcal{X}$ being $[-10, 10]^{2}$ in both dimensions \footnote{We believe that such a toy example is rich enough to demonstrate our proposed approach given the infinite dimensional DRDQN objective.}. There are in total two obstacles that are circular in shape (most simple convex shape assumption made for the sake of simplicity) and a goal region with equal radius namely, $R_{\mathbf{goal}} = R_{\mathbf{obs}}^{(1)} = R_{\mathbf{obs}}^{(2)} = 2$. The robot moves within $\mathcal{X}$ according to the following dynamics,
\begin{equation}
    x_{k+1} = \underbrace{
    \begin{bmatrix}
    1 & 0 \\
    0 & 1
    \end{bmatrix}}
    _{A}x_k + \underbrace{
    \begin{bmatrix}
    1 & 0 \\
    0 & 1
    \end{bmatrix}}
    _{B}u_k + w_k.
\end{equation}
The state of the robot $x_k$ represents the position of the robot in $\mathcal{X} \subset \mathbb{R}^{2}$. The process disturbance $w_k \in \mathbb{R}^{2}$ shifts the position of the robot by a random amount in each axis. For simulation purposes, we considered $10^{4}$ samples of process noise $w$ that were sampled from distributions with zero mean and covariance being equal to $0.15 I_{2}$. The action space $\hat{\mathcal{A}}$ consists of $\left |\hat{\mathcal{A}} \right \vert = 9$ actions where each action is a $\mathbb{R}^2$ vector with unit norm, that represent a step that can be taken in one of the 8 equally spaced radial directions along with a null action. 
The robot takes a step in a specified direction for each action and stays still if a null action is selected. 
The reward function has the constants $\mathbf{r_{travel}} = -0.001$, $\mathbf{r_{goal}} = 1$ and $\mathbf{r_{obs}} = -1$. 
The steepness of the $\tanh$ functions is chosen as $\delta = 0.1$. The continuous reward function that is used here is,
\begin{equation} \label{eqn_cont_reward}
\footnotesize
\begin{aligned}
    r(s^{\prime}) &= \mathbf{r_{travel}} + \frac{\mathbf{r_{goal}}}{2}\left(1 + \tanh \left(\frac{R_{\mathbf{goal}} - \left \|p_{r} - p_g\right \Vert_{2}}{\delta}\right)\right) \\
    &+\frac{\mathbf{r_{obs}}}{2} \sum^{n_{r}}_{j=1} \left(2 + \tanh \left(\frac{\underline{p}(j) - p_{r}(j)}{\delta}\right) + \tanh \left(\frac{p_{r}(j)-\overline{p}(j)}{\delta}\right) \right) \\
    &+ \sum_{i=1}^M \frac{\mathbf{r_{obs}}}{2}\left(1 + \tanh \left(\frac{R_{\mathbf{obs}}^{(i)} - \left \|p_{r} - p_{obs}^{(i)}\right \Vert_{2}}{\delta}\right)\right).
\end{aligned}
\end{equation}

\begin{figure*}[h]
    \centering
    \begin{subfigure}[b]{0.45\textwidth}
        \centering
        \includegraphics[width=\textwidth]{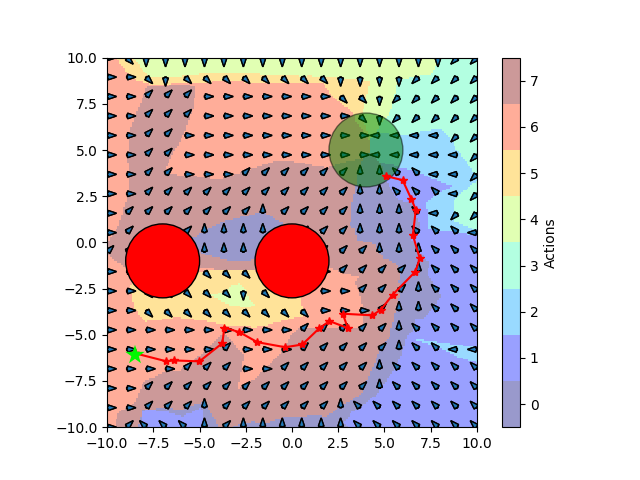}
        \caption{DQN policy}
        \label{fig:dqn_policy}
    \end{subfigure}
    \begin{subfigure}[b]{0.45\textwidth}
        \centering
        \includegraphics[width=\textwidth]{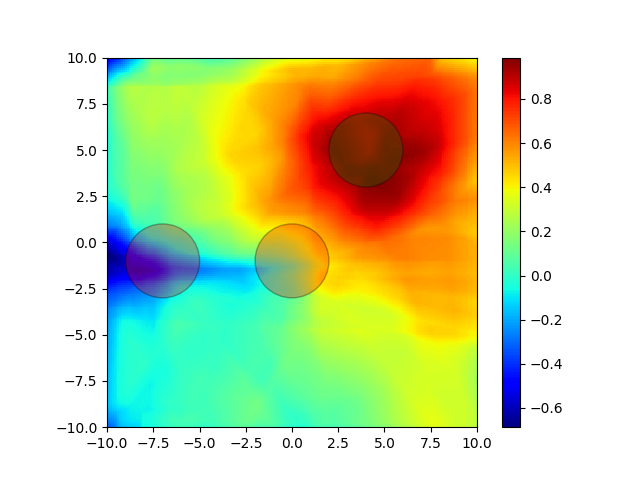}
        \caption{DQN values}
        \label{fig:dqn_values}
    \end{subfigure}
    \begin{subfigure}[b]{0.45\textwidth}
        \centering
        \includegraphics[width=\textwidth]{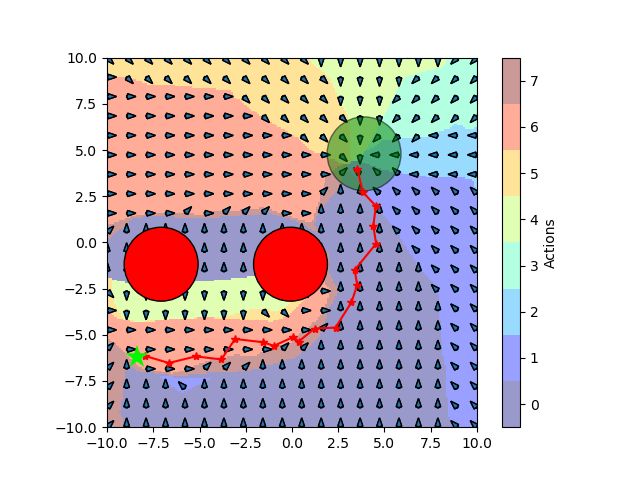}
        \caption{DRDQN policy with $\epsilon_s = 0.067$.}
        \label{fig:drdqn_policy}
    \end{subfigure}
    \begin{subfigure}[b]{0.45\textwidth}
        \centering
        \includegraphics[width=\textwidth]{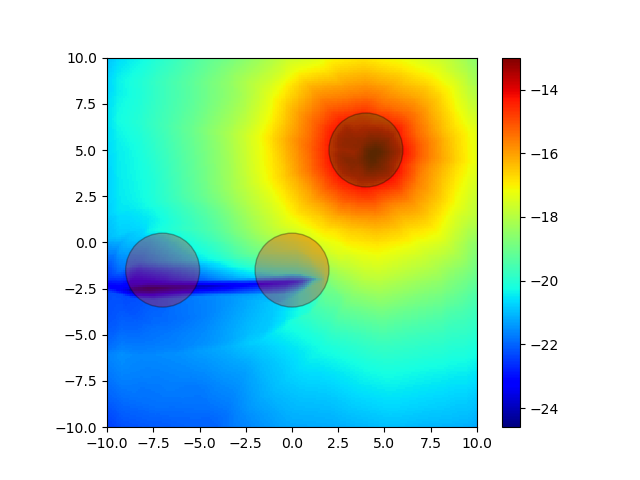}
        \caption{DRDQN values with $\epsilon_s = 0.067$.}
        \label{fig:drdqn_values}
    \end{subfigure}
    \caption{The result of training DQN model and DRDQN model with $\epsilon_s = 0.067$ is shown here. Also, the solution trajectories from a starting position in green star to the goal region in green color avoiding both the red color obstacles are shown in both cases. The plot on the left shows the learned control policies and the one on right depicts the learned Q values.}
    \label{fig:policy_values}
\end{figure*}

\subsection{Discussion of Results}
\begin{table}[h]
\centering
\begin{tabular}{|c|c|c|c|c|c|c|c|}
\hline
& & \multicolumn{6}{|c|}{Noise covariance} \\ 
\cline{3-8} 
& & \multicolumn{2}{|c|}{$\Sigma_w = 0_{2}$} & \multicolumn{2}{|c|}{$\Sigma_w = 0.15 I_{2}$} & \multicolumn{2}{|c|}{$\Sigma_w = 0.3 I_{2}$} \\ 
\cline{3-8} 
& & \multicolumn{2}{|c|}{Reward} & \multicolumn{2}{|c|}{Reward} & \multicolumn{2}{|c|}{Reward} \\ 
\cline{3-8} 
Models & $\epsilon_{s^{\prime}}$ & Mean & Std & Mean & Std & Mean & Std \\ \hline\hline
{\scriptsize DQN} & N/A & 0.636 & 0.545 & 0.662 & 0.514 & 0.627 & 0.573 \\
\hline
{\scriptsize DRDQN} & $0$ & 0.850 & 0.334 & 0.811 & 0.401 & 0.749 & 0.510 \\
\hline
{\scriptsize DRDQN} & $0.067$ & 0.829 & 0.356 & 0.811 & 0.387 & 0.756 & 0.489 \\ 
\hline
\end{tabular}
\caption{The mean and the standard deviation of the total rewards with different noise covariances corresponding to different training models are tabulated here.}
\label{tab:rewards}
\end{table}

\begin{table*}[h]
\centering
\resizebox{\textwidth}{!}{%
\begin{tabular}{|c|c|c|c|c|c|c|c|c|c|c|}
\hline
&  & \multicolumn{3}{|c|}{Reached Goal} & \multicolumn{3}{|c|}{Resulted in Collision} & \multicolumn{3}{|c|}{Wandering Around}   \\ 
\cline{3-11} 
&  & \multicolumn{3}{|c|}{$\Sigma_{w} (\times I_{2})$} & \multicolumn{3}{|c|}{$\Sigma_{w} (\times I_{2})$} & \multicolumn{3}{|c|}{$\Sigma_{w} (\times I_{2})$}   \\ 
\cline{3-11} 
Models & $\epsilon_{s^{\prime}}$ & 0 & 0.15 & 0.3 & 0  & 0.15  & 0.3  & 0 & 0.15 & 0.3 \\ 
\hline
DQN & N/A & 76.7\% & 82.2\% & 81.4\% & 1.5\% & 6.6\% & 6.9\% & 21.7\%  & 13.9\% & 11.9\% \\
\hline
DRDQN & $0$ & 97.9\% & 96.4\% & 93.1\% & 0.7\% & 3.1\% & 6.5\% & 1.3\% & 0.3\% & 0.2\% \\
\hline
DRDQN & $0.067$ & 95.3\% & 95.7\% & 93\% & 0.5\% & 2.4\% & 5.5\% & 4.1\% & 1.8\% & 1.4\% \\ \hline
\end{tabular}%
}
\caption{The percentage of trajectories that reached the goal, resulted in collision and those that did neither are tabulated here for different noise covariances corresponding to different training models.}
\label{tab:results}
\end{table*}
The hyperparameter details of the training results are made available in the supplementary material. The resulting policy and the state values for the trained models can be seen in Figure \ref{fig:policy_values}. 
The arrows represent the action that the policy gives at the respective robot position and goal/obstacle positions. The heatmap represents the same values with color but in higher resolution to better understand the decision boundaries. The figures on the right side of Figure \ref{fig:policy_values} represent the value of each state $s$ which can be computed by $\max_{a \in \mathcal{A}}Q(s,a)$. It can be seen that the rewards propagate from the goal and the obstacles. Further, the resulting learned policy restricts the robot moving between the obstacles and there exists a boundary around the obstacles. When compared with the DQN model, our solution exhibits the most minimum pessimism (risk aversion). This difference is due to the fact that DQN learns the expected rewards, while the DRDQN learns the worst case expected rewards. Due to the noise samples used in DRDQN model with $\epsilon_{s^{\prime}}$ calculated using \eqref{eqn_wasserstein_ball_radius}, learning the policy occurs in less steps compared to the DQN model. However DRDQN can take more time since the computational load is higher for calculating the targets. The DRDQN with $\epsilon_{s^{\prime}} = 0$ is virtually the same as DQN as it learns faster since it calculates the expected values in \eqref{eqn_Bellman} more accurately compared to that of DQN which uses only one sample. 
DQN achieves a lower score overall, since the method only uses one experience per experience replay to train itself, while DRDQN uses the samples provided which results in a much better approximation of (worst case) expected future returns.
The models have been evaluated by running $10^{5}$ episodes each with random goal/obstacle configurations, for three different noise distributions.
As seen in Table \ref{tab:rewards} both versions of DRDQN have a higher average total reward compared to DQN with lower variances. Also in Table \ref{tab:results}, the percentage of trajectories that have reached the goal, collided with an obstacle or border or have not reached the goal or collided, has been provided. It can be inferred that as the covariance of the noise increases, the DRDQN model is able to maintain a low collision rate due to the worst case approximations. Any safe RL algorithm that assumes a particular distribution for $w$ or a bound for $w$ has a greater chance to fail in this setting as the unknown true noise distribution will lead to different transitions than the one assumed.

\section{Conclusion} \label{sec_conclusion}
We proposed a path planning using approximated Wasserstein distributionally robust deep Q-learning approach. Through carefully designed reward function, we showed how to learn safe control policy for uncertain robots operating in an environment. Our numerical simulation results demonstrated our proposed approach. \\

\bibliographystyle{IEEEtran}
\bibliography{references}

\section*{Appendix}
\textbf{Proof of Lemma} \ref{lem:radius}:
Based on Assumption \ref{assump:light_tail}, the risk factor $\beta$ and the radius $\epsilon_{s^{\prime}}$ are related as \cite{dr_learning},
\begin{equation*}
    \mathbb{P}(W_1(\mathbb{P}_{s^{\prime}},\hat{\mathbb{P}}_{s^{\prime}}) \geq \epsilon_{s^{\prime}}) \leq 
    \begin{cases}
    c_1 e^{-c_2 N \epsilon_{s^{\prime}}^{\max(d,2)}}, \quad & \text{if } \epsilon_{s^{\prime}} \leq 1 \\
    c_1 e^{-c_2 N \epsilon_{s^{\prime}}^{a} }, \quad & \text{if } \epsilon_{s^{\prime}} > 1
    \end{cases},
\end{equation*}
where $d$ is the dimension of $s^{\prime} \in \mathbb{R}^d$. In order to obtain $\mathbb{P}(W_1(\hat{\mathbb{P}}_{s^{\prime}}, \mathbb{P}_{s^{\prime}}) \leq \epsilon_{s^{\prime}}) \geq 1 - \beta$, the radius has to be selected as follows
\begin{equation*}
    \epsilon_{s^{\prime}}(\beta) = 
    \begin{cases}
    \left(\frac{\ln(c_1\beta^{-1})}{c_2N}\right) ^ {1/\max(d,2)}, \quad & \text{if } N \geq \frac{\ln(c_1\beta^{-1})}{c_2} \\ 
    \left(\frac{\ln(c_1\beta^{-1})}{c_2N}\right)^{1/a}, \quad & \text{if } N < \frac{\ln(c_1\beta^{-1})}{c_2}
    \end{cases}
\end{equation*}
where $c_1,c_2 \in \mathbb{R}$ are constants that depend on $d$ and $N$. For a discrete nominal distribution $\hat{\mathbb{P}}_{s^{\prime}}$, the radius for the ambiguity set can be calculated as,
\begin{equation*}
    \mathbb{P}(W(\mathbb{P}_{s^{\prime}},\hat{\mathbb{P}}_{s^{\prime}}) \leq \epsilon_{s^{\prime}}) \geq 1 - \underbrace{e^{- \frac{\epsilon^{2}_{s^{\prime}} N}{2 \rho^2}}}_{:= \beta}, 
\end{equation*}
where $\rho$ is the diameter of the support of the true distribution $\mathbb{P}_{s^{\prime}}$. For this paper, $\rho$ is estimated with $\rho \sim \mathbf{diam}(\mathbf{supp}(\hat{\mathbb{P}}_{s^{\prime}}))$. By solving for $\epsilon_{s^{\prime}}$, we can get,
\begin{equation*}
    \epsilon_{s^{\prime}} = \rho \sqrt{ \frac{2}{N} \ln \left( \frac{1}{\beta}\right)}.
\end{equation*}

\textbf{Proof of Lemma} \ref{lem:lip_min}: The Lipschitz approximation can be converted in to a minimization problem by switching the objective function $h$ with $-h$ and multiplying by $-1$. So,
\begin{equation*}
    \sup_{\mathbb{P} \in \mathcal{B}_{s,a}} \mathbb{E}_{s^{\prime} \sim \mathbb{P}}[h(s^{\prime})] = - \inf_{\mathbb{P} \in \mathcal{B}_{s,a}} \mathbb{E}_{s^{\prime} \sim \mathbb{P}}[-h(s^{\prime})]
    \label{eq:sup_inf}
\end{equation*}
If we substitute $h$ with $-h$ in the maximization problem, the Lipschitz approximation becomes,
\begin{equation*}
\begin{split}
    \sup_{\mathbb{P} \in \mathcal{B}_{s,a}} \mathbb{E}_{s^{\prime} \sim \mathbb{P}}[-h(s^{\prime})] &\leq \mathbb{E}_{s^{\prime} \sim \hat{\mathbb{P}}_{s^{\prime}}}[-h(s^{\prime})] + \epsilon_{s^{\prime}} K_{-h} \\
    \iff  \inf_{\mathbb{P} \in \mathcal{B}_{s,a}} \mathbb{E}_{s^{\prime} \sim \mathbb{P}}[h(s^{\prime})] &\geq -\mathbb{E}_{s^{\prime} \sim \hat{\mathbb{P}}_{s^{\prime}}}[-h(s^{\prime})] - \epsilon_{s^{\prime}} K_{-h} \\
    \iff \inf_{\mathbb{P} \in \mathcal{B}_{s,a}} \mathbb{E}_{s^{\prime} \sim \mathbb{P}}[h(s^{\prime})] &\geq \mathbb{E}_{s^{\prime} \sim \hat{\mathbb{P}}_{s^{\prime}}}[h(s^{\prime})] - \epsilon_{s^{\prime}} K_{-h}, \\
\end{split}
\end{equation*}
where $K_{-h}$ is the Lipschitz constant of $-h$. The Lipschitz constant $L_{-h}$ is equivalent to $L_{h}$ since,
\begin{equation*}
    \underbrace{\|-h(x) - (-h(y))\|}_{=\|h(x) - h(y)\|} \leq K_{-h}\|x-y\|, \quad \forall x,y \in \mathbb{R}^{n_{\mathcal{S}}}, x\neq y
\end{equation*}
Hence, $\inf_{\mathbb{P} \in \mathcal{B}_{s,a}} \mathbb{E}_{s^{\prime} \sim \mathbb{P}}[h(s^{\prime})] \geq \mathbb{E}_{s^{\prime} \sim \hat{\mathbb{P}}_{s^{\prime}}}[h(s^{\prime})] - \epsilon_{s^{\prime}} K_{h} $.

\textbf{Proof of Lemma} \ref{lem:lip_sum}: We will prove for the case $N = 2$ and the result for $N > 2$ follows similarly. For the two Lipschitz continuous functions $f_1, f_2$, and $g = f_1 + f_2$, we see that
\begin{equation*}
    \begin{split}
        |g(x) - g(y)| &= |f_1(x) + f_2(x) - f_1(y) - f_2(y)| \\
        &\leq |f_1(x) - f_1(y)| + |f_2(x) - f_2(y)| \\
        &\leq \underbrace{(L_1 + L_2)}_{:= L_{g}} \|x - y\|_2, \quad \forall x,y \in \mathbb{R}^n
    \end{split}
\end{equation*}
For the function $g = \max\{f_1,f_2\}$, $f_1,f_2 : \mathbb{R}^n \to \mathbb{R}$, where $f_1$ and $f_2$ are Lipschitz continuous with constants $K_{f_1},K_{f_2}$, the Lipschitz constants can be defined by,
\begin{equation*}
    \|\nabla f_1\| \leq K_{f_1}, \quad \|\nabla f_2\| \leq K_{f_2}.
\end{equation*}
The gradient of $g$ is,
\begin{equation}
\footnotesize
\|\nabla g\| = 
    \begin{cases}
        \|\nabla f_1\|, & \text{if } f_1(x) > f_2(x) \\
        \|\nabla f_2\|, & \text{if } f_2(x) > f_1(x)
    \end{cases}
    \leq \max \{ \|\nabla f_1\|, \|\nabla f_2\|\}. \nonumber
\end{equation}
Thus $K_g = \max \{ K_{f_1}, K_{f_2}\}$. For a function that is the maximum of $N$ functions, this process can be applied inductively to find $K_g = \max \{K_{f_1},\dots K_{f_N}\}.$

\textbf{Proof of Lemma \ref{lem:lip_tanh}:} For the given scalar functional, its Lipschitz constant corresponds to the maximum magnitude of its slope.
\begin{equation*}
    \begin{split}
        L_f &= \sup_{x \in \mathbb{R}} \left|F^{\prime}(x)\right| \implies F^{\prime}(x) = \frac{A}{2\delta}\left(1 - \tanh^2\left(\frac{x}{\delta}\right) \right) \\
        F^{\prime \prime}(x) &= 0 \implies x = 0 \implies L_f = \left|F^{\prime}(0)\right| = \frac{|A|}{2\delta}.
    \end{split}
\end{equation*}
Similarly, given that $p=[x,y]$ and $g=[g_x,g_y]$ and $f(p) = \frac{A}{2}(1 + \tanh(h(p)))$, where
\begin{equation*}
    h(p) = \frac{R - \|p - g\|_2}{\delta}, \quad \text{and}
\end{equation*}
\begin{equation*}
    \left| \frac{\partial f}{\partial x} \right| = \frac{|A| \, |x - g_x|}{2\delta \|p-g\|_2} \left( 1 - \tanh^2 \left(\frac{R - \|p-g\|_2}{\delta}\right) \right) 
\end{equation*}
The maximum slope occurs at $\tanh(0)$, which corresponds to $\|p-g\|_2 = R$. For a point on this circle such as $|x-g_x| = R$ and $y=g_y$, the slope becomes $\frac{|A|}{2\delta}$. Thus the Lipschitz constant of the function $f$ is $K_f = \frac{|A|}{2\delta}$.
A similar reasoning can be applied for finding the Lipschitz constant of the function $\mathcal{K}(p)$. The maximum value of the derivative of $\mathcal{K}(p)$ occurs at the points where $R_{\mathbf{obs}} = \left(\sum^{n_{r}}_{i=1} \left| p^{(i)}_{r} - \bar{p}^{(i)} \right \vert^{q_{i}} \right)^{\frac{1}{\max\{\mathbf{q}\}}}$. The points that satisfy this create a rectangle that makes up one of the borders of the obstacle. Since our convex polytope is made up of several such stitched together rectangles, it is straightforward to see that $\mathcal{K}_{f} = |A|/(2\delta)$.

\subsection*{Implementation Details}
The simulations were performed on a Dell R530 with 2 Xeon E5-2620 6-core 12-thread CPU's and 132GB of RAM. Simulation time can be improved with a GPU during training. The episodes are limited to $50$ steps. We use a dense neural network with $n_{\mathcal{S}} = 8$ inputs that correspond to the current states to approximate the Q-values. The network has two hidden layers with 150 neurons each that have ReLU activation functions. The network has 9 outputs where each output corresponds to the Q-value for the state action pair. The DQN model has a duelling architecture which is explored in \cite{dueling}. Both models utilize prioritized experience replay in order to prioritize rare experiences during training and the hyperparameters were used as recommended in \cite{priority}. The memory buffer $\mathcal{M}$ is a fixed size buffer that once full, a new experiences replaces the oldest one. During training, collisions do not end the simulation in order to allow the robot to explore further and gain experiences that reach the goal. This does not change anything for the experience replay part since terminal states are handled separately. The probability of taking a random action $\epsilon$ is reduced from 1 to 0.1 linearly for the first $3/4$ of training and kept at 0.1 for the remaining episodes.

\begin{table}[h]
\centering
\begin{tabular}{|c|c|c|c|c|}
\hline
Hyperparameters & DQN& \multicolumn{2}{|c|}{DRDQN} \\
\cline{3-4} 
& & $\epsilon_{s^{\prime}} = 0$ & $\epsilon_{s^{\prime}} = 0.067$              \\ 
\hline \hline
\multicolumn{1}{|c|}{$\gamma$} & 0.9 & 0.9 & 0.9 \\
\hline
\multicolumn{1}{|c|}{$\eta$} & $10^{-4}$ & $10^{-4}$ & $10^{-4}$ \\
\hline
\multicolumn{1}{|c|}{Steps per episode} & 50 & 50 & 50 \\
\hline
\multicolumn{1}{|c|}{Total steps} & $10^7$ & $2.4 \times 10^{5}$ & $2.4 \times 10^{5}$ \\
\hline
\multicolumn{1}{|c|}{$N_{batch}$} & 32 & 32 & 32 \\
\hline
\multicolumn{1}{|c|}{$\beta$} & N/A & N/A & 0.1 \\
\hline
\multicolumn{1}{|c|}{$N$ (samples)} & N/A & $10^{4}$ & $10^{4}$ \\
\hline
\multicolumn{1}{|c|}{$\Gamma$} & 5000 & 1500 & 1500 \\
\hline
\multicolumn{1}{|c|}{$|\mathcal{M}|$} & 5000 & 5000 & 5000 \\
\hline
\multicolumn{1}{|c|}{$\epsilon$} & $1 \to 0.1$ & $1 \to 0.1$ & $1 \to 0.1$\\ 
\hline
\end{tabular}
\caption{Hyperparameters used in the simulation results.}
\label{tab:hyperparameters}
\end{table}

\end{document}